\documentclass[letterpaper,10pt,journal,twoside]{IEEEtran}
\usepackage{amsmath, amssymb, amsfonts, amsthm}
\usepackage{algorithmic}
\usepackage{algorithm}
\usepackage{array}
\usepackage[caption=false]{subfig}
\usepackage{textcomp}
\usepackage{stfloats}
\usepackage{url}
\usepackage{verbatim}
\usepackage{graphicx}
\usepackage{eso-pic}
\usepackage{cite}
\usepackage{capt-of}

\usepackage{booktabs}
\usepackage{multirow}
\usepackage{multicol}
\usepackage[table]{xcolor}
\usepackage{pifont}
\usepackage{bm}
\usepackage{makecell} 
\usepackage{hyperref}
\hypersetup{hypertex=true,
colorlinks=true,
linkcolor=black,
anchorcolor=black,
citecolor=black,
urlcolor=black}

\newcommand{\Frame}[1]{\underrightarrow{\mathcal{F}}_{#1}}
\newcommand{\xmark}{\ding{55}}
\newcommand{\cmark}{\ding{51}}

\newcommand{\bestcell}[1]{\textbf{#1}}
\newcommand{\secondcell}[1]{\underline{#1}}
\newcommand{\pmsep}{\hspace{0.22em}\(\pm\)\hspace{0.22em}}
\newcommand{\bestleft}[1]{\textbf{#1}}
\newcommand{\bestright}[1]{\textbf{#1}}
\newcommand{\secondleft}[1]{\underline{#1}}
\newcommand{\secondright}[1]{\underline{#1}}

\newtheorem{proposition}{Proposition}

\theoremstyle{remark}
\newtheorem{remark}{Remark}

\begin{document}

\title{A$^2$I-Calib++: A Kinematics-Driven Anti-noise Active Calibration Framework for Whole-Body Multi-IMU Systems on Legged Robots}

\author{Fangyu Jiang\IEEEauthorrefmark{2}, Chaoran Xiong\IEEEauthorrefmark{2},~\IEEEmembership{Graduate Student Member,~IEEE}, Kehui Ma,\\
Zhen Sun, Zeyu Zhang, and Ling Pei$^{\ast}$,~\IEEEmembership{Senior Member,~IEEE}\thanks{\fontsize{8}{8.25}\selectfont Received: July 30, 2026; Conditionally accepted: October 8, 2026. This paper was recommended for publication by Editor Lucia Pallottino upon evaluation of the Associate Editor and Reviewers' comments. This work was supported by the National Natural Science Foundation of China (Grant 62673319) and Shanghai QiYuan Innovation Foundation. This work was also partially supported by ``Research and Development of Robust Localization and Guidance Technologies for Complex Industrial Environments'' funded by Openmind Smart Robotics Technology Co., Ltd.}\thanks{\fontsize{8}{8.25}\selectfont $^{\dagger}$The first two authors contributed equally. $^{\ast}$Corresponding author: Ling Pei. The authors are with Shanghai Key Laboratory of Navigation and Location Based Services, School of Automation and Intelligent Sensing, Shanghai Jiao Tong University, Shanghai 200240, China (e-mail: \{jiangfangyu; sjtu4742986; khma0929; zhensun; zhang-zeyu; ling.pei\}@sjtu.edu.cn). Ling Pei is also the Chief Scientist of Jiangxi Zhilian Advanced Technology Research Co., Ltd.}\thanks{\fontsize{8}{8.25}\selectfont Author-prepared revised manuscript.}}
\IEEEaftertitletext{\vspace{-9.5mm}}

\markboth{IEEE Robotics and Automation Letters. Preprint Version. Conditionally Accepted October, 2026}{Jiang \MakeLowercase{\textit{et al.}}: A$^2$I-Calib++: Whole-Body Multi-IMU Active Calibration}

\hypersetup{pdftitle={A2I-Calib++: A Kinematics-Driven Anti-noise Active Calibration Framework for Whole-Body Multi-IMU Systems on Legged Robots},
pdfauthor={Fangyu Jiang, Chaoran Xiong, Kehui Ma, Zhen Sun, Zeyu Zhang, Ling Pei}}

\AddToShipoutPictureFG*{
\AtPageLowerLeft{
\put(0,10){\makebox[\paperwidth][c]{
\parbox[b]{\textwidth}{\fontsize{7.5}{8.5}\selectfont\noindent
\textcopyright\ 2026 IEEE. Personal use of this material is permitted. Permission from IEEE must be obtained for all other uses, in any current or future media, including reprinting/republishing this material for advertising or promotional purposes, creating new collective works, for resale or redistribution to servers or lists, or reuse of any copyrighted component of this work in other works.}
}}}
}

\maketitle
\pagestyle{headings}

\begin{abstract}
Recently, multi-node inertial measurement unit (multi-IMU) systems have shown promise for robust odometry in legged robots.
Multi-IMU fusion can improve legged robot odometry but requires accurate spatiotemporal calibration among the sensors.
Existing multi-IMU calibration methods, designed for aerial or handheld platforms, face challenges on legged robots because routine locomotion may provide insufficient excitation, while foot slips, impacts, and vibrations introduce measurement disturbances.
To address these challenges, we propose A$^2$I-Calib++, a kinematics-driven anti-noise active calibration framework for body-mounted IMUs covered by the robot kinematic model. It estimates temporal offsets, extrinsic rotations and translations, and inertial biases through temporal correlation search followed by analytic spatial and bias estimation, without artificial calibration targets or complex parameter tuning for the target IMU. A$^2$I-Calib++ comprises three core modules:
1) a Noise-Isotropy-Inspired Trajectory supported by a finite-horizon conditioning proposition that generates sufficient excitation to limit noise sensitivity; 2) a constraint-aware Adversarial Motion Priors controller that combines the anti-noise excitation with a rolling-contact reward for contact-aware execution; and 3) a contact-constrained virtual IMU that fuses pre-calibrated calf IMUs and joint kinematics under a rolling-contact model to construct the calibration reference.
Furthermore, evaluations on both simulation and real-world quadruped platforms demonstrate notably reduced noise sensitivity and simulation calibration errors compared to conventional methods, together with repeatable real-world calibration and improved proprioceptive odometry accuracy for legged robots.

\end{abstract}

\begin{IEEEkeywords}

Legged robots, calibration and identification, multi-node inertial measurement units, reinforcement learning, proprioceptive odometry.

\end{IEEEkeywords}

\section{Introduction}
\IEEEPARstart{R}{eliable} state estimation of legged robots\cite{wisth2023vilens,wu2024360,ou2024legkilo} constitutes a fundamental prerequisite for executing complex tasks in embodied navigation\cite{xiong2025en}. Recently, multi-IMU-based odometry has garnered considerable attention due to cost effectiveness and power efficiency\cite{li2025mio,yang2023mipo,xavier2023mipse}.

Distributing IMUs across a legged robot and fusing them with joint encoders
can improve base state estimation
\cite{yang2023mipo,xavier2023mipse,katayama2025learning}. However,
spatiotemporal misalignment among inertial nodes introduces inconsistent
kinematic constraints and systematic odometry drift
\cite{xiong2025tonvio}, making multi-IMU calibration essential.

Meeting this requirement on legged platforms remains difficult. Passive
targetless methods\cite{li2024targetless,chen2025ikalibr} relying on human intervention suffer from insufficient kinematic excitation, leading to enlarged sensitivity to IMU measurement noise
caused by foot slips, impacts, and vibrations.
Existing observability-aware methods
\cite{lv2022observability,wang2025targetfree} mainly assess collected data
or provide qualitative motion guidance, rather than directly designing
excitation to improve the numerical conditioning of the calibration
system. A$^2$I-Calib\cite{IROS} actively calibrates temporal offsets,
extrinsic rotations and gyroscope biases of calf-mounted IMUs, but
extending active calibration to body-mounted IMUs additionally requires
extrinsic-translation and accelerometer-bias estimation.

\begin{figure}[!t]
\centering
\includegraphics[width=1\linewidth]{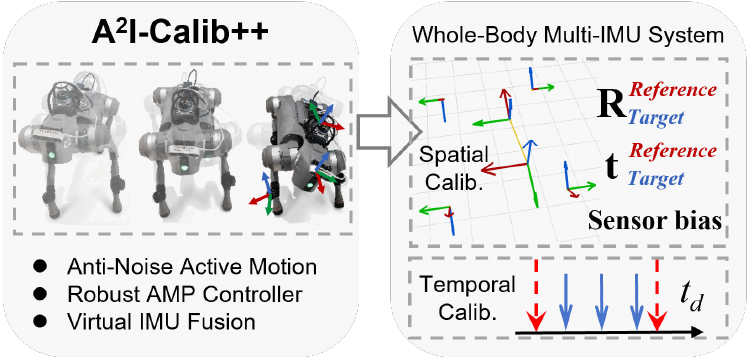}
\caption{A$^2$I-Calib++ performs kinematics-driven spatiotemporal calibration of body IMUs on legged robots using anti-noise calibration motion, AMP-based execution, and a contact-constrained VIMU reference, thereby improving whole-body multi-IMU calibration and downstream proprioceptive odometry.}
\vspace{-4mm}
\label{fig:teaser}
\vspace{-2mm}
\end{figure}

\begin{table}[!t]
\caption{Comparison of Calibration Capabilities of Different Methods}
\vspace{-2mm}
\label{tab:a2i_capability}
\centering
\renewcommand{\arraystretch}{1.05}
\setlength{\tabcolsep}{2.6pt}
\begin{tabular*}{\linewidth}{@{\extracolsep{\fill}}cccccccc@{}}
\toprule
\textbf{Method} & \textbf{Target} & \textbf{Mode} & $t_d$ & $\mathbf{R}$ & $\mathbf{p}$ & $\mathbf{b}_g$ & $\mathbf{b}_a$ \\
\midrule
iKalibr~\cite{chen2025ikalibr} & Body IMUs & Passive Calibration & \cmark & \cmark & \cmark & \cmark & \cmark \\
A$^2$I-Calib~\cite{IROS} & Calf IMUs & Active Calibration & \cmark & \cmark & \xmark & \cmark & \xmark \\
A$^2$I-Calib++ & Body IMUs & Active Calibration& \cmark & \cmark & \cmark & \cmark & \cmark \\
\bottomrule
\multicolumn{8}{@{}l@{}}{\makebox[0pt][l]{\parbox{0.95\linewidth}{\footnotesize \textsuperscript
{*} In iKalibr, IMU biases are jointly estimated only in configurations involving exteroceptive sensors\cite{chen2025ikalibr}.}}}
\end{tabular*}
\vspace{-2mm}
\end{table}

To address these challenges, we present A$^2$I-Calib++, a kinematics-driven anti-noise active calibration framework for whole-body multi-IMU systems on legged robots.
Our main contributions are as follows:
\begin{enumerate}
\item An active calibration framework that recovers the full set of spatiotemporal parameters with temporal correlation search and analytic spatial estimation (search and solve), requiring no additional artificial targets or complex parameter tuning.
\item A kinematics-driven anti-noise mechanism that combines an anti-noise reference excitation generated by the Noise Isotropy Proposition (NIP) with a contact-constrained virtual IMU (VIMU) reference formed by multi-leg fusion to reduce calibration noise sensitivity.
\item A constraint-aware Adversarial Motion Priors (AMP) controller that learns the reference motion's state-transition distribution and incorporates a rolling-contact task reward to generate contact-aware calibration motions for legged-robot execution.
\item Comprehensive validation across simulation and real-world quadruped platforms, demonstrating notable reductions in
noise sensitivity, high simulation calibration accuracy, repeatable real-world calibration, and lower proprioceptive odometry drift compared with a state-of-the-art (SOTA) calibration method.
\end{enumerate}

\section{Related Work}

Multi-IMU estimators for legged robots commonly use a central IMU for state propagation and kinematically coupled peripheral IMUs for corrective updates\cite{yang2023mipo,xavier2023mipse,katayama2025learning}. Their fusion therefore requires consistent temporal offsets, extrinsic transformations, and sensor biases. Otherwise, spatiotemporal misalignment introduces inconsistent kinematic constraints and systematic odometry drift.

Existing multi-IMU calibration methods can be grouped into exteroceptive-aided and pure-inertial approaches. Rehder et al.\cite{rehder2016extending} extended Kalibr with camera observations for multi-IMU extrinsic and intrinsic calibration. Light Detection and Ranging (LiDAR) measurements have similarly been used to assist multi-IMU calibration\cite{li2024targetless}, while Hartzer et al.\cite{hartzer2023online} proposed online calibration using visual-inertial odometry. Pure-inertial methods instead exploit trajectory correlation, but their accuracy remains sensitive to measurement noise and excitation quality\cite{qiu2021real,lee2022extrinsic}. Although active calibration can deliberately improve excitation, existing trajectory-design methods mainly address industrial manipulators rather than agile quadrupeds\cite{huang2008determination,yang2025industrial}.

Deep reinforcement learning (RL) has substantially expanded quadruped motor skills\cite{peng2022amp,arm2024pedipulate,cheng2024extreme}, yet its use for sensor calibration remains underexplored. A$^2$I-Calib\cite{IROS} employs an RL controller but is limited to calf-mounted IMUs and estimates only their temporal offsets, rotations and gyroscope biases, leaving body-mounted IMUs, translation and accelerometer biases unaddressed. A$^2$I-Calib++ therefore relies on A$^2$I-Calib as a preliminary calf-IMU calibration step to recover the spatiotemporal parameters of body-mounted IMUs. The resulting capability extension and comparison with iKalibr are illustrated in Tab.~\ref{tab:a2i_capability}.

\section{Problem Formulation}

\subsection{IMU Measurement Model and Kinematics}
\label{subsec:imu_model}

For a multi-IMU sensor suite, let $\Frame{c^i}$ and $\Frame{r}$ denote the frames of the $i$-th target IMU and the reference sensor, respectively. The target measurements follow
\begin{equation}
\tilde{\boldsymbol{\omega}}^{c^i} = \boldsymbol{\omega}^{c^i} + \mathbf{b}_g^{c^i} + \mathbf{n}_g^{c^i} ,\quad
\tilde{\mathbf{a}}^{c^i} = \mathbf{a}^{c^i} + \mathbf{b}_a^{c^i} + \mathbf{n}_a^{c^i},
\end{equation}
where $\boldsymbol{\omega}^{c^i}$ and $\mathbf a^{c^i}$ are the true angular velocity and specific force, $\mathbf b_g^{c^i}$ and $\mathbf b_a^{c^i}$ are biases, and $\mathbf n_g^{c^i}$ and $\mathbf n_a^{c^i}$ are Gaussian white-noise terms. Transmission and synchronization delays yield the temporal offset $t_d^i=t_c^i-t_r$. The rigid-body relationships are
\begin{equation}
\boldsymbol{\omega}^r = \mathbf{R}^r_{c^i} \boldsymbol{\omega}^{c^i} ,\quad
    \mathbf{v}^r = \mathbf{R}^r_{c^i} \mathbf{v}^{c^i}(t) - \lfloor \boldsymbol{\omega}^r \rfloor_\times \mathbf{t}^r_{r\to c^i},
\end{equation}
where $\mathbf R^r_{c^i}\in SO(3)$ maps vectors from $\Frame{c^i}$ to $\Frame{r}$, $\mathbf t^r_{r\to c^i}\in\mathbb R^3$ is the translation from the reference origin to the target origin expressed in $\Frame{r}$, and $\lfloor\cdot\rfloor_\times$ is the skew-symmetric operator. The calibration problem estimates
\begin{equation}
    \mathcal{X}^i = \{t_d^i, \mathbf{R}^r_{c^i}, \mathbf{b}^{c^i}_g, \mathbf{t}^r_{r\to c^i}, \mathbf{b}^{c^i}_a\},
\end{equation}
where $t_d^i$, $\mathbf b_g^{c^i}$, and $\mathbf b_a^{c^i}$ are assumed constant over the short observation window in this work.

\subsection{Conditioning Objective for Calibration Trajectory}
\label{subsec:calib_traj}

For a fixed, full-column-rank system $\mathbf H\mathbf x=\mathbf y$, where $\mathbf H\in\mathbb R^{m\times n}$, $m\geq n$, and $\mathbf x,\mathbf y\ne\mathbf0$, the spectral condition number and its observation-perturbation bound are \cite{golub2013matrix,sauer2011numerical}
\begin{equation}
\kappa_2(\mathbf H)
=\frac{\sigma_{\max}(\mathbf H)}{\sigma_{\min}(\mathbf H)},\;
\frac{\|\delta\mathbf x\|}{\|\mathbf x\|}
\leq
\kappa_2(\mathbf H)
\frac{\|\delta\mathbf y\|}{\|\mathbf y\|}.
\label{eq:condition_error_bound}
\end{equation}
\vspace{-0.mm}
where $\sigma_{\max}(\mathbf H)$ and $\sigma_{\min}(\mathbf H)$ are the largest and smallest singular values of $\mathbf H$, while $\delta\mathbf x$ and $\delta\mathbf y$ are the parameter and observation perturbations, respectively. $\kappa_2(\mathbf H)$ bounds the worst-case relative amplification from observation noise to parameter error. A large condition number permits strong noise amplification and leads to noise-sensitive estimates, whereas a small condition number tightens this bound. The best-conditioned case $\kappa_2(\mathbf H)=1$ minimizes the amplification.

Accordingly, let $\mathbf H(\mathcal T)$ denote the trajectory-dependent calibration matrix. We seek an anti-noise calibration trajectory $\mathcal T^*$ that minimizes the noise sensitivity of the calibration problem:
\vspace{-0.6mm}
\begin{equation}
\label{eq:optimal_traj}
    \mathcal{T}^* = \arg\min_{\mathcal{T}} \kappa_2\!\left(\mathbf H(\mathcal T)\right).
\end{equation}
\vspace{-1.0cm}

\section{Methodology}
\vspace{-0.6mm}

\begin{figure*}
\centering
\includegraphics[width=0.99\linewidth]{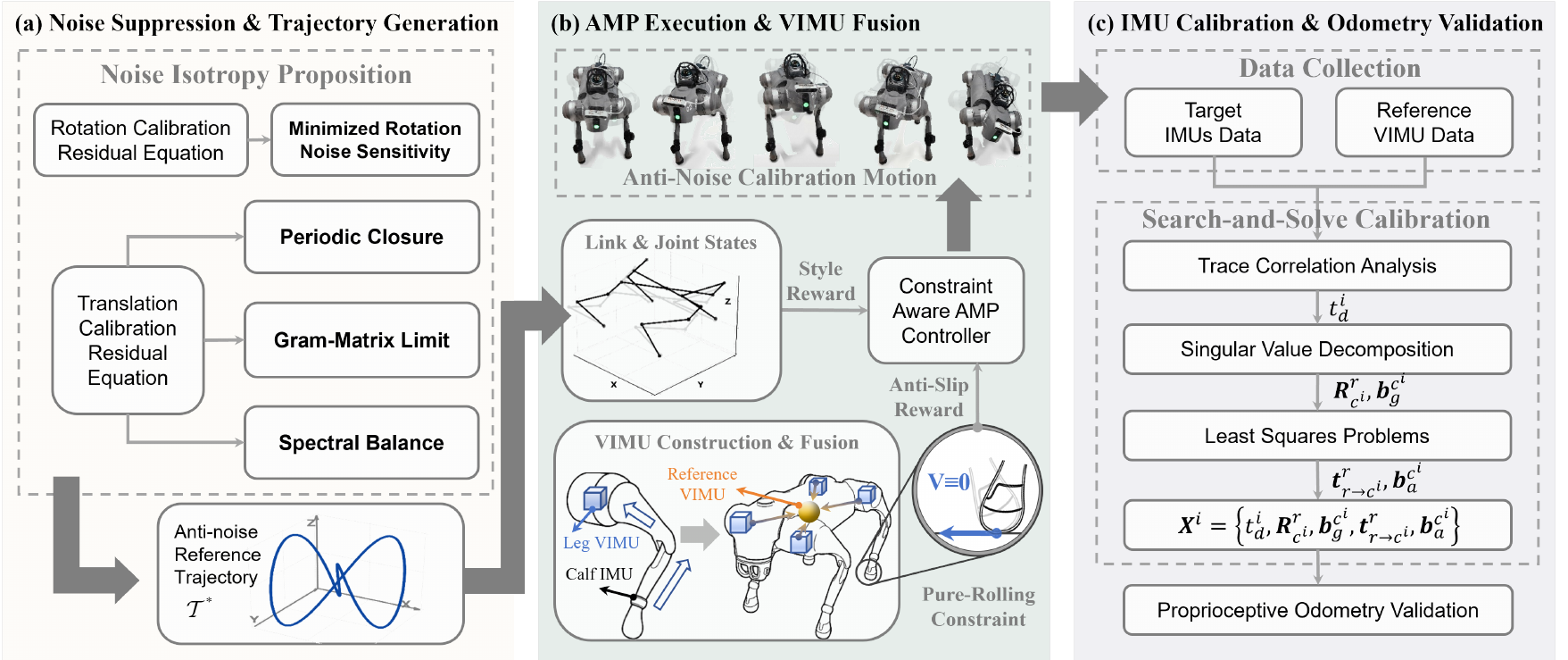}
\caption{System overview of A$^2$I-Calib++. (a) summarizes the Noise-Isotropy Proposition (NIP) and the resulting anti-noise reference excitation. (b) shows AMP-based contact-aware motion execution and multi-leg VIMU reference construction under a rolling-contact model. (c) presents search-and-solve parameter estimation and downstream proprioceptive-odometry evaluation.}
\label{fig:system_overview}
\vspace{-2mm}
\end{figure*}

\subsection{System Overview}
\label{subsec:system_overview}

The A$^2$I-Calib++ pipeline is shown in Fig. \ref{fig:system_overview}. First, the Noise-Isotropy Proposition (NIP) characterizes the finite-horizon conditioning of orthogonal harmonic excitation and guides the construction of an anti-noise calibration trajectory. A preliminary A$^2$I-Calib step estimates the calf-IMU temporal offsets and extrinsic rotations required for kinematic propagation. The constraint-aware AMP policy is trained with a style reward derived from the reference motion's state-transition distribution and a rolling-contact task reward to generate a contact-aware excitation. Under the rolling-contact model, measurements from the pre-calibrated calf IMUs and joint encoders are propagated to leg VIMUs and fused into a reference VIMU, reducing sensitivity to disturbances from any single leg. The target and reference sequences are then processed by the search-and-solve calibration procedure to estimate $\mathcal X$. Finally, the estimated parameters are used to correct the IMU measurements, and their downstream effect is evaluated through proprioceptive odometry against a LiDAR-IMU odometry (LIO) reference.

\vspace{-0.9mm}
\subsection{Search-and-Solve IMU Calibration}
\label{subsec:calib_solver}

For clarity, we omit the target index $i$ in this subsection. To recover the parameter set $\mathcal X$, we decouple the
spatiotemporal calibration into two sequential problems.

\subsubsection{Temporal and Rotational Calibration}
\label{subsubsec:t_R_bg}
The angular-velocity sample cross-covariance matrix is approximated by
\begin{equation}
\mathbf\Sigma_{cr}
=
\frac{1}{N-1}\sum_{n=1}^{N}
\left(\tilde{\boldsymbol\omega}_n^c-\overline{\boldsymbol\omega}^c\right)
\left(\tilde{\boldsymbol\omega}_n^r-\overline{\boldsymbol\omega}^r\right)^\top,
\end{equation}
where $\overline{\boldsymbol\omega}^c$ and $\overline{\boldsymbol\omega}^r$ are the empirical means of the target and reference measurements, respectively. For a candidate time delay $t_d$, let ${}^{t_d}\tilde{\boldsymbol\omega}_n^r$ denote the shifted reference measurement, and let ${}^{t_d}\mathbf\Sigma_{rr}$, ${}^{t_d}\mathbf\Sigma_{rc}$, and ${}^{t_d}\mathbf\Sigma_{cr}$ denote the corresponding covariance matrices. The temporal offset maximizes the trace correlation coefficient \cite{qiu2021real}:
\begin{equation}
\label{eq:td}
t_d^\star
=
\arg\max_{t_d}
\sqrt{\frac{1}{3}\operatorname{Tr}\!\left(
\left({}^{t_d}\mathbf\Sigma_{rr}\right)^{-1}
{}^{t_d}\mathbf\Sigma_{rc}
\mathbf\Sigma_{cc}^{-1}
{}^{t_d}\mathbf\Sigma_{cr}
\right)}.
\end{equation}
\vspace{-0.4mm}
By applying singular value decomposition, the extrinsic rotation and gyroscope bias are estimated by \cite{sorkine2017least}
\begin{equation}
\begin{gathered}{}^{t_d^\star}\mathbf\Sigma_{cr}=\mathbf U\mathbf\Sigma\mathbf V^\top,\enspace \mathbf R_c^r=\mathbf V\operatorname{diag}\!\left(1,1,\det(\mathbf V\mathbf U^\top)\right)\mathbf U^\top,\\\mathbf b_g^c=\overline{\boldsymbol\omega}^c-(\mathbf R_c^r)^\top\overline{\boldsymbol\omega}^r.\end{gathered}
\label{eq:R_bg}
\end{equation}
\vspace{-0.4mm}
Here the empirical means in Eq. \eqref{eq:R_bg} are computed from the temporally aligned sequences.

\vspace{-0.mm}
\subsubsection{Translation and Accelerometer Bias Calibration}
Define the reference-to-target translation in the target frame as
\begin{equation}
\mathbf p^c_{r\to c}
=
(\mathbf R_c^r)^\top\mathbf t^r_{r\to c}.
\end{equation}
The rigid-body relation in Section \ref{subsec:imu_model} then gives
\begin{equation}
(\mathbf R_c^r)^\top\mathbf v_n^r-\mathbf v_n^c
=
-\lfloor\boldsymbol\omega_n^c\rfloor_\times
\mathbf p^c_{r\to c}.
\end{equation}
Assuming the calibration motion starts from a strict standstill, i.e., $\mathbf v^r(t_0)=\mathbf 0$ and $\boldsymbol\omega^r(t_0)=\mathbf 0$, the accelerometer bias $\mathbf b_a^c$ induces a velocity drift governed by the target-frame integration matrix
\begin{equation}
\label{eq:Qc}
\mathbf Q_n^c
=
\left(\mathbf R_c^w(t_n)\right)^\top
\int_{0}^{t_n}\mathbf R_c^w(\tau)d\tau,
\end{equation}
where $\mathbf R_c^w(t)$ maps vectors from $\Frame{c}$ to the world frame. The velocity residual is therefore
\begin{equation}
\label{eq:vel_residual}
\mathbf y_{v,n}
=
(\mathbf R_c^r)^\top\tilde{\mathbf v}_n^r
-
\tilde{\mathbf v}_n^c
=
-\lfloor\tilde{\boldsymbol\omega}_n^c\rfloor_\times
\mathbf p^c_{r\to c}
-
\mathbf Q_n^c\mathbf b_a^c.
\end{equation}
Stacking $N$ measurements gives
\begin{equation}
\begin{aligned}
\mathbf H_v
&=
\begin{bmatrix}
-\lfloor\tilde{\boldsymbol\omega}_1^c\rfloor_\times&-\mathbf Q_1^c\\
\vdots&\vdots\\
-\lfloor\tilde{\boldsymbol\omega}_N^c\rfloor_\times&-\mathbf Q_N^c
\end{bmatrix},&
\mathbf x_v
&=
\begin{bmatrix}
\mathbf p^c_{r\to c}\\
\mathbf b_a^c
\end{bmatrix},\\
\mathbf H_v\mathbf x_v&=\mathbf y_v,&
\mathbf\Lambda&=\mathbf H_v^\top\mathbf H_v.
\end{aligned}
\end{equation}
\vspace{-0.7mm}
The least-squares estimate is obtained from the pseudoinverse of $\mathbf H_v$, and the reported translation is recovered as $\mathbf t^r_{r\to c}=\mathbf R_c^r\mathbf p^c_{r\to c}$. Together with Eqs. \eqref{eq:td} and \eqref{eq:R_bg}, this yields the parameter set $\mathcal X$.

\vspace{-0mm}
\subsection{Noise Isotropy Inspired Trajectory (NIT) Generation}
\label{subsec:theorem}
\vspace{-0.3mm}
For the trajectory objective in Eq. \eqref{eq:optimal_traj}, the rotational and translational condition numbers are $\kappa_{\mathbf{R}}=\kappa_2(\mathbf\Sigma_{cr})=\kappa_2(\mathbf\Sigma_{cc})$ and $\kappa_{\mathbf{t}}=\kappa_2(\mathbf H_v)=\sqrt{\kappa_2(\mathbf\Lambda)}$, respectively.
Orthogonal harmonic excitation yields the following Noise Isotropy Proposition, leading to the generation of an anti-noise calibration trajectory $\mathcal T^*$.

\vspace{-0.5mm}
\begin{proposition}
\label{prop:nit}
For fixed physical scales $p_0,b_0>0$ and any $T>0$, let $\mathbf P=\operatorname{diag}(p_0\mathbf I_3,b_0\mathbf I_3)$ and consider the phase-zero body-rate family
\vspace{-0.3mm}
\begin{equation}
\label{eq:nit_trajectory}
\boldsymbol\omega_k(t)
=
C
\begin{bmatrix}
\sin\!\frac{2k \pi}{T}t&
\sin\!\frac{4k \pi}{T}t&
\sin\!\frac{6k \pi}{T}t
\end{bmatrix}^{\top},
\end{equation}
\vspace{-0.4mm}
where $C>0$, $k\in\mathbb N^+$, and $0\le t\le T$. Under ideal, exactly aligned continuous-time execution over complete periods, if $\mathbf R_k(0)=\mathbf I_3$ and $C=(b_0/p_0)T/\sqrt3$, then $\mathbf\Sigma_{cc,k}$ and $\mathbf\Lambda_{11,k}$ are isotropic for every finite $k$, and
\vspace{-0.3mm}
\begin{equation}
\kappa_{\mathbf{R}}
=
\kappa_2(\mathbf\Lambda_{11,k})
=1,
\ 
\inf_{k\in\mathbb N^+}
\sqrt{\kappa_2(\mathbf P^\top\mathbf\Lambda_k\mathbf P)}
=1.
\label{eq:nit_infimum}
\end{equation}
\vspace{-0.4mm}
The normalized infimum is approached as $k\to\infty$ but is unattainable by any finite nonzero member.
\end{proposition}

\begin{proof}
With $t$ implicit, define the corrected target-frame design and its $3\times3$ block Gram matrix by
\begin{equation}
\begin{gathered}\mathbf Q_k^c=\mathbf R_k^\top\int_0^t\mathbf R_k(\tau)d\tau,\quad \mathbf H_{v,k}=\begin{bmatrix}-\lfloor\boldsymbol\omega_k\rfloor_\times&-\mathbf Q_k^c\end{bmatrix},\\\mathbf\Lambda_k=\int_0^T\mathbf H_{v,k}^\top\mathbf H_{v,k}dt=\begin{bmatrix}\mathbf\Lambda_{11,k}&\mathbf\Lambda_{12,k}\\\mathbf\Lambda_{12,k}^\top&\mathbf\Lambda_{22,k}\end{bmatrix}.\end{gathered}
\label{eq:nit_matrices}
\end{equation}

\textit{a) Rotational and lever-arm isotropy:}
Complete-period sine orthogonality gives zero mean and equal axis energy. Together with
$\lfloor\boldsymbol\omega\rfloor_\times^\top
\lfloor\boldsymbol\omega\rfloor_\times
=\|\boldsymbol\omega\|^2\mathbf I_3
-\boldsymbol\omega\boldsymbol\omega^\top$,
this yields
\begin{equation}
\mathbf\Sigma_{cc,k}
=(C^2/2)\mathbf I_3,
\qquad
\mathbf\Lambda_{11,k}
=TC^2\mathbf I_3.
\label{eq:nit_exact_blocks}
\end{equation}
Multiplication by the orthogonal $\mathbf R_c^r$ preserves the singular values of $\mathbf\Sigma_{cc,k}$, proving both isotropy claims.

\textit{b) Micro-period closure and high-frequency limit:}
Within every micro-period of length $T/k$, the second half of each sinusoidal body-rate component is the sign-reversed, time-reversed copy of its first half. Specifically, for $j=0,\ldots,k-1$ and $0\le s\le\frac{T}{2k}$,
\begin{equation}
\boldsymbol\omega_k\!\left(\frac{(2j+1)T}{2k}+s\right)
=-\boldsymbol\omega_k\!\left(\frac{(2j+1)T}{2k}-s\right).
\label{eq:nit_time_reversal}
\end{equation}
The attitude therefore retraces the first-half rotation in reverse and returns to its initial value at the end of the micro-period. By uniqueness of the attitude ODE, this gives $\mathbf R_k((j+1)T/k)=\mathbf R_k(jT/k)$. Since $\mathbf R_k(0)=\mathbf I_3$, induction yields $\mathbf R_k(jT/k)=\mathbf I_3$. Within each micro-period,
\begin{equation}
\sup_{t\in[0,T]}
\operatorname{dist}_{SO(3)}(\mathbf R_k(t),\mathbf I_3)
\le\sqrt3CT/k
\xrightarrow{k\to\infty}0.
\end{equation}
Thus, as $k\to\infty$, $\mathbf R_k(t)\to\mathbf I_3$ and $\mathbf Q_k^c(t)\to t\mathbf I_3$ on $[0,T]$.

\textit{c) Full Gram-matrix limit:}
From Eq. \eqref{eq:nit_matrices} and the uniform limits above, with the following arrows taken as $k\to\infty$,
\begin{equation}
\begin{aligned}
\dot{\mathbf Q}_k^c
&=-\lfloor\boldsymbol\omega_k\rfloor_\times\mathbf Q_k^c+\mathbf I_3,\\
\mathbf\Lambda_{12,k}
&=\mathbf Q_k^c(T)-T\mathbf I_3\longrightarrow\mathbf0,\\
\mathbf\Lambda_{22,k}
&=\int_0^T(\mathbf Q_k^c)^\top\mathbf Q_k^c\,dt
\longrightarrow\frac{T^3}{3}\mathbf I_3.
\end{aligned}
\end{equation}
Consequently, $\mathbf\Lambda_k
\longrightarrow
\operatorname{diag}\!\left(
TC^2\mathbf I_3,\frac{T^3}{3}\mathbf I_3
\right).$

\textit{d) Spectral balance:}
Eigenvalue continuity gives
\begin{equation}
\lim_{k\to\infty}\kappa_2(\mathbf P^\top\mathbf\Lambda_k\mathbf P)
=
\frac{\max\{p_0^2TC^2,b_0^2T^3/3\}}
{\min\{p_0^2TC^2,b_0^2T^3/3\}},
\end{equation}
so $C=(b_0/p_0)T/\sqrt3$ proves Eq. \eqref{eq:nit_infimum}. Raw-SI values set $p_0=1~\mathrm m$ and $b_0=1~\mathrm{m/s^2}$.

\textit{e) Non-attainment:}
If a finite nonzero member made $\mathbf P^\top\mathbf\Lambda_k\mathbf P$ a scalar identity, then $\mathbf\Lambda_{12,k}=\mathbf0$ and $\mathbf Q_k^c(T)=T\mathbf I_3$. Since $\mathbf R_k(T)=\mathbf I_3$, this gives $\int_0^T\mathbf R_k(t)dt=T\mathbf I_3$ and hence
\begin{equation}
\int_0^T\|\mathbf R_k(t)-\mathbf I_3\|_F^2dt
=
6T-2\operatorname{tr}\!\left(\int_0^T\mathbf R_k(t)dt\right)
=0.
\end{equation}
The nonnegative integrand therefore vanishes identically, giving $\mathbf R_k(t)\equiv\mathbf I_3$ and hence $\boldsymbol\omega_k(t)\equiv\mathbf0$, which contradicts $C>0$. A rank-deficient member instead has infinite condition number, so no finite nonzero member attains Eq. \eqref{eq:nit_infimum}.
\end{proof}

\begin{remark}
For full-rank systems under sufficiently small perturbations, let $\epsilon_\Sigma=\|\Delta\mathbf\Sigma_{cr}\|_F/\|\mathbf\Sigma_{cr}\|_F$. Higham's polar-factor bound \cite{higham1986computing} and Eq. \eqref{eq:condition_error_bound} give
\vspace{-0.5mm}
\begin{equation}
\begin{aligned}
\frac{\|\Delta\mathbf R_c^r\|_F}{\|\mathbf R_c^r\|_F}
&\le
\sqrt3(1+\sqrt2)\kappa_{\mathbf{R}}
\epsilon_\Sigma
\!+\mathcal O(\epsilon_\Sigma^2),\\
\|\Delta\mathbf x_v\|/\|\mathbf x_v\|
&\le
\kappa_{\mathbf{t}}
\|\Delta\mathbf y_v\|/\|\mathbf y_v\|
\!.
\end{aligned}
\label{eq:calibration_perturbation}
\end{equation}
Under unit conversion, $\mathbf P$ and $\mathbf H_v$ scale inversely. Hence $\mathbf H_v\mathbf P$ and its condition number are unchanged, and the proposed trajectory is physically invariant.
\end{remark}

\begin{remark}
The trajectory used for both simulation and real-world deployment,
expressed in $\mathrm{rad/s}$ for $0\le t\le3.2~\mathrm{s}$ is
\begin{equation}
\boldsymbol\omega_2(t)=\begin{bmatrix}\sin(1.25\pi t)&\sin(2.5\pi t)&\sin(3.75\pi t)\end{bmatrix}^{\top}.
\label{eq:nit_experimental_trajectory}
\end{equation}
To ensure safe execution on real robots, simulation tests with the same kinematic model established $C=1.0~\mathrm{rad/s}$ as a safe amplitude. Under fixed SI scaling, $C=T/\sqrt3$ suggests a one-period window of approximately $1.73~\mathrm{s}$. We choose the nearby $T_{\mathrm{per}}=1.6~\mathrm{s}$, giving the fundamental frequency $0.625~\mathrm{Hz}$. For calibration, this waveform is repeated twice, giving $T=3.2~\mathrm{s}$ to reduce the influence of the less-stable startup transient.
For the resulting theoretical trajectory, continuous-time evaluation gives $\kappa_{\mathbf{R}}=1.000$ and $\kappa_{\mathbf{t}}\approx1.872$.
In deployment, finite sampling and frequency leakage have negligible conditioning effects: across 200--1000~Hz processing and up to 100-ms truncation, $\kappa_{\mathbf{t}}$ changed by less than $1\%$.
\end{remark}

\subsection{Contact-Constrained Virtual IMU Fusion}
\label{subsec:vimu}
To construct a reliable calibration reference, we fuse calf IMUs pre-calibrated by A$^2$I-Calib into a contact-constrained VIMU.
For the $m$-th leg, let $\Frame{m}$ and $\Frame{q}$ denote its calf-link and calf-IMU frames, respectively. The preliminary stage supplies the time offset $\widehat{t}_d^q$, IMU-to-calf rotation $\widehat{\mathbf{R}}_q^m$, and gyroscope bias $\widehat{\mathbf{b}}_g^q$.
These parameters yield the time-aligned, bias-corrected calf angular velocity used in kinematic propagation:
\begin{equation}
\boldsymbol{\omega}_{calf}^{m}(t)
=
\widehat{\mathbf{R}}_q^m
\left[
\widetilde{\boldsymbol{\omega}}^q
\left(t+\widehat{t}_d^q\right)
-\widehat{\mathbf{b}}_g^q
\right].
\label{eq:calf_angular_velocity_correction}
\end{equation}
\subsubsection{Kinematic Propagation under Pure-Rolling Constraint}
During active calibration on level ground, the grounded feet are assumed to maintain pure-rolling contact, yielding zero instantaneous linear velocity $\mathbf{v}_{contact}^m \equiv \mathbf{0}$, as illustrated in Fig. \ref{fig:vimu}. Thus the linear velocity of the robot's calf is given by
\begin{equation}
    \mathbf{v}_{calf} = \mathbf{v}_{contact} + \boldsymbol{\omega}_{foot} \times \mathbf{r}_{foot} = \boldsymbol{\omega}_{foot} \times \mathbf{r}_{foot}.
\end{equation}
For the $m$-th supporting leg of a robot, initialized with $\mathbf{v}_{calf}^m$ and given the joint position $\mathbf{q}^{m}$, joint velocity $\dot{\mathbf{q}}^{m}$ and the calibrated calf angular velocity $\boldsymbol{\omega}^m_{calf}$ in \eqref{eq:calf_angular_velocity_correction}, the robot's base angular and linear velocities are recursively propagated upward through the forward kinematics chain $\text{FK}$:

\begin{equation}
\begin{aligned}
    \boldsymbol{\omega}^{v^m} &= \text{FK}_\omega \left( \boldsymbol{\omega}^m_{calf}, \mathbf{q}^m, \dot{\mathbf{q}}^m \right) \\
    \mathbf{v}^{v^m} &= \text{FK}_v \left( \mathbf{v}_{calf}^m, \boldsymbol{\omega}^m_{calf}, \mathbf{q}^m, \dot{\mathbf{q}}^m \right),
\end{aligned}
\end{equation}
where $\boldsymbol{\omega}^{v^m}$ and $\mathbf{v}^{v^m}$ are the $m$-th leg VIMU's measurements mapped to the trunk as shown in Fig. \ref{fig:vimu}.

\begin{figure*}[!t]
\begin{minipage}[t]{0.49\textwidth}
\vspace{0pt}
\centering
\includegraphics[width=0.9\linewidth]{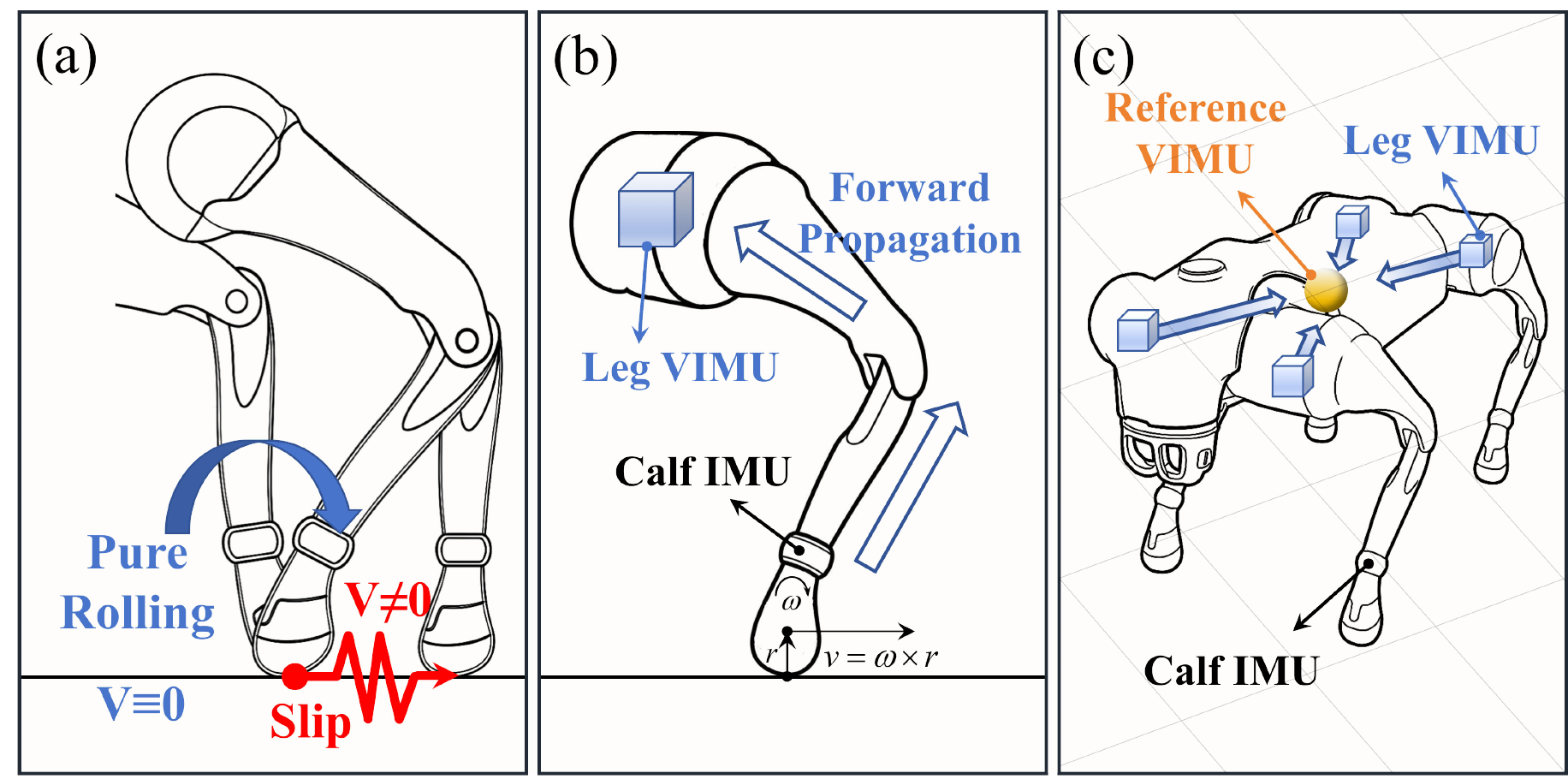}
\captionof{figure}{Kinematic propagation and VIMU fusion. (a) compares the pure-rolling contact with foot slip. (b) shows the propagation of forward kinematics to construct a leg VIMU. (c) demonstrates the fusion of multiple VIMUs into a reference VIMU.}
\label{fig:vimu}
\end{minipage}
\hfill
\begin{minipage}[t]{0.49\textwidth}
\vspace{0pt}
\centering
\footnotesize
\renewcommand{\arraystretch}{1.0}
\captionof{table}{Observations, Actions, and Rewards of the AMP Controller}
\vspace{-2mm}
\label{tab:obs}
\begin{tabular}{lcccc}
\toprule
\textbf{Features} & \textbf{Sym.} & \textbf{Dim.} & \textbf{Priv.} & \textbf{Policy} \\
\midrule
Foot contact mode       & $\mathbf{c}_{foot}$     & 4  & \cmark & \xmark \\
Base linear velocity    & $\mathbf{v}_B$          & 3  & \cmark & \xmark \\
Base angular velocity   & $\boldsymbol{\omega}_B$ & 3  & \cmark & \cmark \\
Projected gravity       & $\mathbf{g}_B$          & 3  & \cmark & \cmark \\
Joint positions         & $\mathbf{q}$            & 12 & \cmark & \cmark \\
Joint velocities        & $\dot{\mathbf{q}}$      & 12 & \cmark & \cmark \\
Previous actions        & $\mathbf{u}_{t-1}$      & 12 & \cmark & \cmark \\
Motion phase            & $\boldsymbol{\phi}$     & 2  & \cmark & \cmark \\
\midrule
Joint position deviations & $\mathbf{u}_t$          & 12 & \multicolumn{2}{c}{\textit{Policy Output}} \\
\midrule
Anti-slip task reward   & $r_t$                   & 1  & \multicolumn{2}{c}{\textit{Weight: $w_t = 0.2$}} \\
Style reward            & $r_s$                   & 1  & \multicolumn{2}{c}{\textit{Weight: $w_s = 0.8$}} \\
\bottomrule
\end{tabular}
\end{minipage}
\vspace{-3mm}
\end{figure*}

\subsubsection{Multi-Leg VIMU Fusion and Noise Suppression}
The $M$ supporting-leg VIMU estimates are fused by arithmetic averaging:
\vspace{-0.6mm}
\begin{equation}
    \tilde{\boldsymbol{\omega}}^r = \frac{1}{M} \sum_{m=1}^M \boldsymbol{\omega}^{v^m}, \quad
    \tilde{\mathbf{v}}^r = \frac{1}{M} \sum_{m=1}^M \mathbf{v}^{v^m}.
\end{equation}
\vspace{-0.mm}
Each leg VIMU combines calf-IMU measurements, joint encoders, forward kinematics, and the contact model, so the effective leg errors can be correlated. For either the angular or linear channel, let $\mathbf e_m$ denote the error of leg $m$, $\mathbf\Sigma_{ij}=\operatorname{Cov}(\mathbf e_i,\mathbf e_j)$, and $\bar{\mathbf\Sigma}=M^{-1}\sum_i\mathbf\Sigma_{ii}$. The fused error $\bar{\mathbf e}=M^{-1}\sum_m\mathbf e_m$ has covariance
\vspace{-1mm}
\begin{equation}
\mathbf\Sigma_{\bar e}
=\frac{1}{M^2}\sum_{i,j=1}^M\mathbf\Sigma_{ij}
=\frac{\bar{\mathbf\Sigma}}{M}
+\frac{1}{M^2}\sum_{i\ne j}\mathbf\Sigma_{ij}.
\label{eq:vimu_general_covariance}
\end{equation}
\vspace{-0.mm}
For mutually uncorrelated, equal-covariance leg errors, $\mathbf\Sigma_{ij}=\mathbf0$ for $i\ne j$ and $\mathbf\Sigma_{ii}=\mathbf\Sigma$, so Eq. \eqref{eq:vimu_general_covariance} reduces to $\mathbf\Sigma_{\bar e}=\mathbf\Sigma/M$. Four-leg fusion then gives a $75\%$ covariance-trace reduction \cite{yu2023fast}. Thus, averaging suppresses uncorrelated errors, while cross-covariance limits the attainable reduction.

If $s_m\leq M\|\mathbf y_v\|e_{\mathbf t,\max}/(\sqrt N\,\kappa_{\mathbf t}\|\mathbf x_v\|)$, where $s_m$ is leg $m$'s window-RMS tangential slip speed and $e_{\mathbf t,\max}$ its allowable translation-error increment, pure-rolling contact is adequate.
As the condition number bounds the worst case, this is an upper bound, not a strict threshold.

\vspace{-1mm}
\subsection{Constraint-Aware Adversarial Motion Prior Controller}
\label{subsec:amp}

Directly tracking the theoretical NIT can expose a gap between the ideal reference and feasible robot motion, causing foot slip, contact-induced noise, and instability. We therefore employ a constraint-aware AMP controller \cite{peng2022amp} to learn the reference state-transition distribution and retain its anti-noise excitation, while encouraging pure-rolling contact to reduce disturbances in the kinematic reference.

\vspace{-0.5mm}
\subsubsection{Reference Dataset}
Using a Unitree Go2 kinematic model, a traditional controller tracks the reference trajectory and collects robot states as the reference dataset, as shown in Fig. \ref{fig:system_overview}(b). The following reward design balances excitation retention against contact violations introduced by traditional controllers.

\vspace{-0.5mm}
\subsubsection{Reward Design} 
The style reward $r_s$ encourages generated state transitions to follow the reference distribution and preserve its anti-noise excitation, while the task reward $r_t$ penalizes translational slip to encourage pure-rolling contact:
\vspace{-1mm}
\begin{equation}
    r_{t} = \sum_{m \in \mathcal{C}} \exp\left(-\|\mathbf{v}_{foot}^m - \boldsymbol{\omega}_{foot}^m \times \mathbf{r}_{foot}\|/\sigma_{slip}\right)
\end{equation}
\vspace{-0mm}
where $\mathcal{C}$ denotes the set of grounded feet, and $\mathbf{r}_{foot}$ points from the contact point to the foot center as shown in Fig. \ref{fig:vimu}. The residual vanishes under ideal spherical rolling and grows with departures from that relation. The total reward is
$r = w_{s} r_{s} + w_{t} r_{t}$,
where $w_s$ and $w_t$ balance the two terms.

\vspace{-0mm}
\subsubsection{Observation and Action Spaces}
The asymmetric actor-critic uses privileged critic inputs only during training. Its 12-D action specifies target joint-position deviations tracked by a 50 Hz Proportional-Derivative (PD) controller. The deployable observations and complete spaces are summarized in Tab. \ref{tab:obs}.
\subsubsection{Domain Randomization}
During training, we randomize ground friction, base mass, and joint PD gains, and apply random pushes to improve robustness to friction variations and external disturbances.

\begin{figure}[!b]
\centering
\includegraphics[width=0.85 \linewidth]{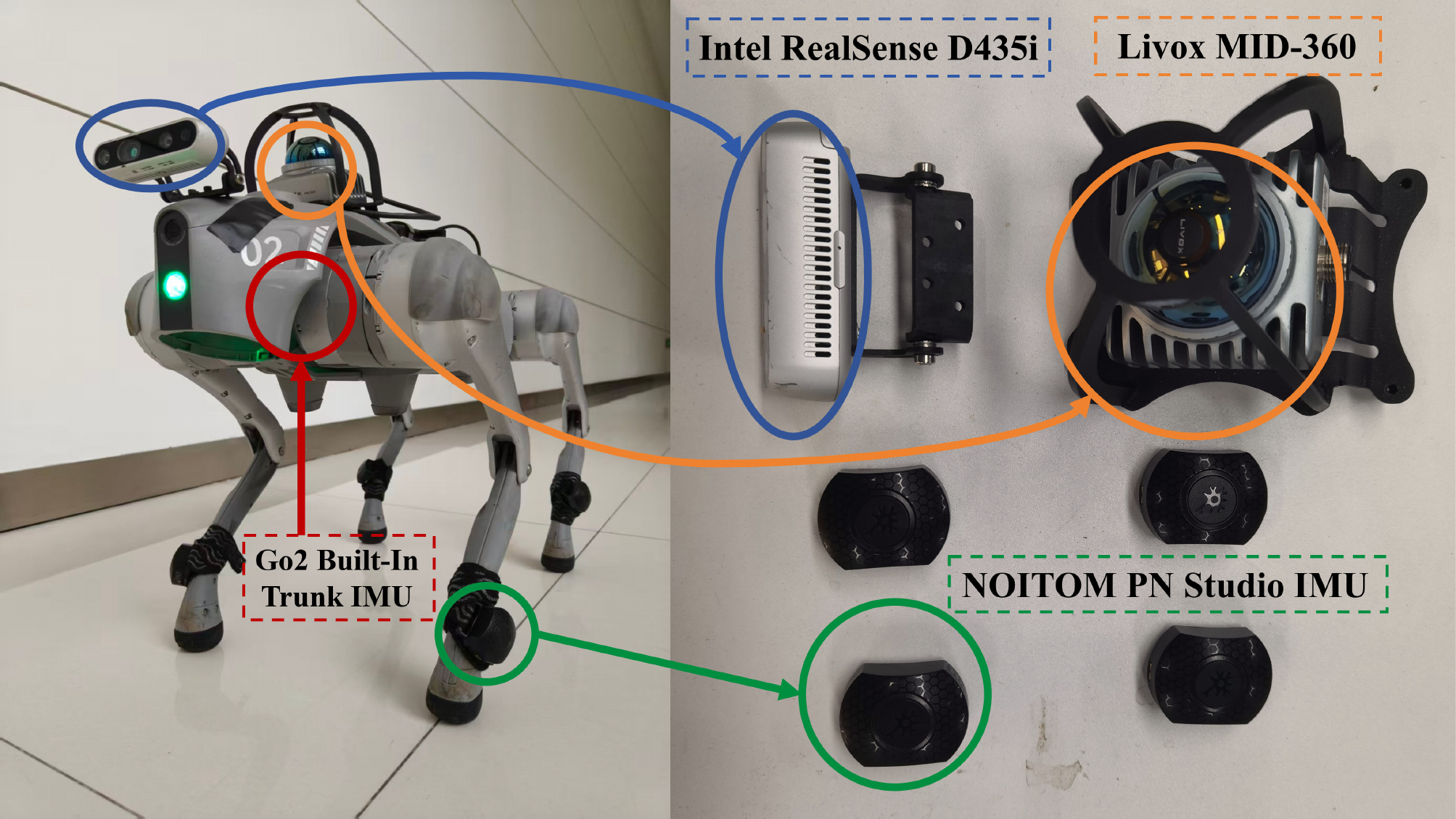}\caption{Unitree Go2 Platform with Noitom PN Studio suite, Intel RealSense D435i and Livox MID-360.}
\label{fig:platform}
\end{figure}

\begin{table}[!t]
\caption{Comparison of Algorithm Runtime}
\vspace{-2mm}
\label{tab:runtime}
\centering
\renewcommand{\arraystretch}{1.38}
\setlength{\tabcolsep}{3pt}
\resizebox{\linewidth}{!}{
\begin{tabular}{c@{\hspace{12pt}}ccccc@{}c@{}c}
\toprule
\multirow{2}{*}{\textbf{Duration}} &
\multicolumn{5}{c}{\textbf{A$^2$I-Calib++}} &
\hspace{12pt} & \textbf{iKalibr} \\
\cmidrule(lr){2-6}\cmidrule(lr){8-8}
& Search & SVD & LS & Others & Total & & Total \\
\midrule
3.200 & 0.036 & 0.005 & 0.004 & 2.000 & $2.045\pm0.068$ & & $2.941\pm2.105$ \\
6.400 & 0.043 & 0.005 & 0.006 & 2.030 & $2.084\pm0.070$ & & $3.555\pm1.123$ \\
\bottomrule
\end{tabular}}
\begin{minipage}{0.98\linewidth}
\footnotesize\setlength{\baselineskip}{8.5pt} \textsuperscript{*} All entries are in seconds. Total reports sample mean $\pm$ sample standard deviation; the component columns report means.
\end{minipage}
\par\medskip
\renewcommand{\arraystretch}{1.0}
\setlength{\tabcolsep}{6pt}
\caption{Camera-IMU Calibration Sensitivity to Calf-IMU Residuals in Simulation Experiments}
\vspace{-2mm}
\label{tab:precalib_sensitivity}
\centering
\begin{tabular}{c*{5}{c}}
\toprule
\multirow{2}{*}{\textbf{Calf Residual}} & \multicolumn{5}{c}{\textbf{Calibration Error / Displacement}} \\
\cmidrule(l){2-6}
& $|\delta t_d|$
& $\|\delta\boldsymbol{\theta}\|$
& $\|\delta\mathbf p\|$
& $\|\delta\mathbf b_g\|$
& $\|\delta\mathbf b_a\|$ \\
\midrule
Zero Residual & $0.110$ & $0.022$ & $0.292$ & $0.0026$ & $0.414$ \\
$|\delta t_d|=2$ & $0.252$ & $0.059$ & $0.045$ & $0.0007$ & $0.027$ \\
$\|\delta\boldsymbol{\theta}\|=3$ & $0.032$ & $0.546$ & $0.463$ & $0.0021$ & $0.026$ \\
$\|\delta\mathbf b_g\|=0.01$ & $0.004$ & $0.012$ & $0.007$ & $0.0050$ & $0.020$ \\
All Residuals & $0.252$ & $0.550$ & $0.465$ & $0.0055$ & $0.042$ \\
\bottomrule
\end{tabular}
\par\vspace{0.6mm}
\begin{minipage}{0.98\linewidth}
\footnotesize\setlength{\baselineskip}{8.5pt}
Time offset $\delta t_d$ in (\textit{ms}), rotation $\delta\boldsymbol{\theta}$ in (\textit{deg}), translation $\delta\mathbf p$ in (\textit{cm}), gyroscope bias $\delta\mathbf b_g$ in (\textit{rad/s}), and accelerometer bias $\delta\mathbf b_a$ in (\textit{cm/s}$^2$).
\end{minipage}
\vspace{-2.5mm}
\end{table}

\vspace{-3mm}
\section{Experiments}

\subsection{Experimental Setup}
\label{subsec:exp_setup}

\subsubsection{Development of Experimental Platforms}

The AMP controller is trained in Isaac Gym \cite{rudin2022learning} on an NVIDIA RTX 4060 GPU, while Gazebo evaluation uses four 100-Hz calf IMUs, a 500-Hz trunk IMU, and a 100-Hz camera IMU with additive Gaussian white noise. Real-world experiments use a Unitree Go2 with four 96-Hz Noitom calf IMUs, a 500-Hz trunk IMU, a 200-Hz Intel RealSense D435i camera IMU, and a Livox MID-360 LiDAR, as shown in Fig. \ref{fig:platform}.

\vspace{+1.mm}
\subsubsection{Comparison Baselines}

We separate three experimental factors: reference motion, execution controller, and calibration solver. \textbf{NIT + AMP (A$^2$I-Calib++)} executes the proposed NIT reference with AMP, while Sinusoid + AMP (w/o NIT) uses a sinusoidal reference to isolate NIT's contribution. NIT w/o AMP (w/o AMP) and Sinusoid w/o AMP (Sinusoid) execute the same references without AMP. Walking is a conventional forward-and-turning locomotion baseline. For robot safety, only the two AMP-based motions and Walking are deployed on hardware.

At the solver level, the proposed Search-and-Solve procedure is compared with the IMU-only calibration of iKalibr \cite{chen2025ikalibr}. Table \ref{tab:runtime} reports their algorithm runtime for the two calibration durations, where the Others category includes resampling and derivative computation, calf-IMU preliminary calibration, window preprocessing, and base-VIMU construction. The runtime therefore changes little when the calibration motion duration doubles because the preprocessing dominates the computation.
A$^2$I-Calib++ is faster in these experiments using lightweight correlation search, SVD, and least squares, whereas iKalibr performs iterative nonlinear optimization.

\vspace{-1.1mm}
\subsection{Sensitivity to Preliminary Calf-IMU Calibration}
\label{subsec:precalib_sensitivity}

To quantify sensitivity to calf-IMU calibration errors, Table \ref{tab:precalib_sensitivity} reports how leg-wise residuals in $t_d$, $\mathbf R$, and $\mathbf b_g$ change all five estimated parameters of the camera IMU. The residuals are normalized to the prescribed four-leg RMS levels. Zero Residual retains the simulated Gaussian IMU noise without additional preliminary-calibration residuals, while the remaining rows report the mean parameter displacements caused by individual or simultaneous residuals. The full cross-parameter results capture how preliminary errors propagate through VIMU construction to the target-IMU calibration.

\begin{table}[t]
\caption{Foot-Slip, Contact, VIMU Metrics, and Motion Quality of Different Calibration Motions}
\vspace{-2mm}
\label{tab:contact_vimu_validation}
\centering
\renewcommand{\arraystretch}{1.05}
\setlength{\tabcolsep}{1.5pt}
\resizebox{\linewidth}{!}{
\begin{tabular}{ccc c@{\,/\,}c c@{\,/\,}c c@{\,/\,}c >{\centering\arraybackslash}m{0.80cm}}
\toprule
\multirow{2}{*}{\textbf{Motion}}
& \textbf{Slip}
& \textbf{Contact}
& \multicolumn{4}{c}{\textbf{VIMU Error\textsuperscript{*}}}
& \multicolumn{2}{c}{\textbf{Condition}}
& \textbf{Track.} \\
& ($\mathrm{m/s}$)
& ($\%$)
& \multicolumn{2}{c}{$\|\Delta\boldsymbol{\omega}\|$}
& \multicolumn{2}{c}{$\|\Delta\mathbf{v}\|$}
& \multicolumn{2}{c}{$\kappa_{\mathbf{R}}/\kappa_{\mathbf{t}}$}
& \\
\midrule
\multicolumn{10}{c}{\textit{Simulation Experiments: Ground Truth Available}} \\
\midrule
A$^2$I-Calib++ & \secondcell{0.005} & \bestcell{100.00} & \bestcell{0.005} & 45.3 & \secondcell{0.003} & \secondcell{41.4} & \secondcell{1.46} & \bestcell{1.64} & 1.89 \\
w/o NIT & \bestcell{0.001} & \bestcell{100.00} & \secondcell{0.005} & 37.3 & \bestcell{0.002} & 27.5 & 275.70 & 9.85 & 0.98 \\
w/o AMP & 0.021 & 83.25 & 0.008 & \secondcell{48.7} & 0.012 & \bestcell{45.9} & \bestcell{1.19} & \secondcell{2.14} & \secondcell{0.10} \\
Sinusoid & 0.032 & 76.15 & 0.010 & 42.9 & 0.020 & 35.4 & 337.70 & 26.47 & \bestcell{0.08} \\
Walking & 1.431 & 60.79 & 0.490 & \bestcell{49.1} & 1.013 & 36.6 & 14.72 & 3.40 & -- \\
\midrule\midrule
\multicolumn{10}{c}{\textit{Real-World Experiments: Proxy Estimates}} \\
\midrule
A$^2$I-Calib++ & \secondcell{0.038} & \bestcell{100.00} & \secondcell{0.158} & 26.5 & \secondcell{0.043} & \secondcell{28.3} & \bestcell{7.40} & \bestcell{3.36} & \secondcell{1.44} \\
w/o NIT & \bestcell{0.035} & \bestcell{100.00} & \bestcell{0.108} & \secondcell{27.8} & \bestcell{0.033} & 16.9 & 83.82 & 16.08 & \bestcell{1.05} \\
Walking & 0.302 & 60.48 & 1.659 & \bestcell{32.6} & 0.496 & \bestcell{40.9} & \secondcell{8.34} & \secondcell{8.96} & -- \\
\bottomrule
\end{tabular}}
\begin{minipage}{0.98\linewidth}
\footnotesize\setlength{\baselineskip}{8.5pt}
\textsuperscript{*} Each VIMU entry reports four-leg RMSE / reduction from one-leg RMSE. Angular-velocity and linear-velocity errors are in (\textit{rad/s}) and (\textit{m/s}), respectively; reductions are in (\textit{\%}), and Tracking is in (\textit{rad/s}).
\par\textsuperscript{$\dagger$} Throughout this paper, the best values are shown in boldface, whereas the second-best values are underlined.
\end{minipage}
\vspace{-2.5mm}
\end{table}

For input levels of $2~\mathrm{ms}$, $3^\circ$, and $0.01~\mathrm{rad/s}$, the corresponding camera-IMU time-offset, rotation, and gyroscope-bias shifts are $0.252~\mathrm{ms}$, $0.546^\circ$, and $0.0050~\mathrm{rad/s}$, respectively. All remain below the injected residual magnitudes, indicating limited sensitivity to leg-dependent preliminary-calibration errors.

\vspace{-1.1mm}
\subsection{Contact-Aware Motion and VIMU Validation}
\label{subsec:contact_vimu_validation}

Table \ref{tab:contact_vimu_validation} evaluates VIMU fusion under different calibration motions. Slip measures the horizontal pure-rolling violation, Contact is the mean per-foot contact duty, and VIMU Error reports the four-leg vector RMSE and its reduction from the one-leg RMSE for angular and linear velocity. Condition gives $\kappa_{\mathbf{R}}/\kappa_{\mathbf{t}}$, while Tracking is the angular-velocity execution-gap vector RMSE against the reference trajectory. Simulation uses ground truth, whereas the real-world Slip and VIMU Error entries are proxy estimates because ground truth is unavailable.

In simulation, AMP reduces the four-leg angular- and linear-velocity errors for both reference motions. On hardware, the AMP-based motions yield smaller reference-consistency proxies than Walking. A large relative reduction can accompany a large remaining error when the single-leg error is high. The fused error must therefore be considered together with the reduction percentage.

\begin{table*}[!t]
\renewcommand{\arraystretch}{1.3}
\setlength{\tabcolsep}{4pt}
\caption{Comparison with Conventional Motions of Trajectory Quality and  Calibration Results in Simulation Experiments}
\vspace{-2mm}
\label{tab:sims_ablation}
\centering
\resizebox{0.99\linewidth}{!}{
\begin{tabular}{c|*{3}{c@{\pmsep}c}|*{4}{c@{\pmsep}c}>{\centering\arraybackslash}p{0.72cm}@{\pmsep}>{\centering\arraybackslash}p{0.72cm}}
\toprule
\multicolumn{1}{c|}{\textbf{Settings}} & \multicolumn{6}{c|}{\textbf{Trajectory Quality (Mean $\pm$ STD)}} & \multicolumn{10}{c}{\textbf{Calibration Accuracy (RMSE $\pm$ STD)}} \\
\cmidrule(r){1-1} \cmidrule(lr){2-7} \cmidrule(l){8-17}

\multicolumn{1}{c|}{Motion} & \multicolumn{2}{c}{Trace Corr.} & \multicolumn{2}{c}{Rot. Cond.} & \multicolumn{2}{c|}{Trans. Cond.} & \multicolumn{2}{c}{Time Offset} & \multicolumn{2}{c}{Extri. Rot.} & \multicolumn{2}{c}{Gyro. Bias} & \multicolumn{2}{c}{Extri. Trans.} & \multicolumn{2}{c}{Accel. Bias} \\

\multicolumn{1}{c|}{Style} & \multicolumn{2}{c}{$r_{max}$ (\textit{\%}) $\uparrow$} & \multicolumn{2}{c}{$\kappa_{\mathbf{R}}$ $\downarrow$} & \multicolumn{2}{c|}{$\kappa_{\mathbf{t}}$ $\downarrow$} & \multicolumn{2}{c}{$|\delta t_d|$ (\textit{ms}) $\downarrow$} & \multicolumn{2}{c}{$\|\delta \boldsymbol{\theta}\|$ (\textit{deg}) $\downarrow$} & \multicolumn{2}{c}{$\|\delta \mathbf{b}_g\|$ (\textit{rad/s}) $\downarrow$} & \multicolumn{2}{c}{$\|\delta \mathbf{p}\|$ (\textit{cm}) $\downarrow$} & \multicolumn{2}{c}{$\|\delta \mathbf{b}_a\|$ (\textit{cm/s}$^2$) $\downarrow$} \\
\midrule

\textbf{A$^2$I-Calib++} & \secondleft{99.9} & \secondright{0.0} & \secondleft{1.46} & \bestright{0.01} & \bestleft{1.640} & \secondright{0.006} & \bestleft{0.19} & \bestright{0.05} & \bestleft{0.177} & \bestright{0.062} & 0.002 & 0.001 & \bestleft{0.287} & \secondright{0.066} & 0.66 & 0.26 \\
w/o NIT      & 99.3 & 0.1 & 280.64 & 3.73 & 9.862 & 0.066 & 0.38 & 0.15 & 0.675 & 0.266 & \secondleft{0.001} & \secondright{0.001} & 3.390 & 0.342 & 0.54 & 0.20 \\
w/o AMP      & \bestleft{99.9} & \bestright{0.0} & \bestleft{1.19} & \secondright{0.01} & \secondleft{2.033} & \bestright{0.004} & \secondleft{0.24} & \secondright{0.05} & \secondleft{0.191} & \secondright{0.065} & \bestleft{0.001} & \bestright{0.001} & \secondleft{0.335} & \bestright{0.044} & \bestleft{0.47} & \secondright{0.19} \\
Sinusoid     & 97.3 & 0.1 & 349.89 & 3.07 & 29.492 & 0.602 & 8.38 & 2.19 & 0.410 & 0.129 & 0.002 & 0.001 & 2.267 & 1.231 & \secondleft{0.50} & \bestright{0.14} \\
Walking      & 90.8 & 11.5 & 14.77 & 2.01 & 3.461 & 0.240 & 107.72 & 101.51 & 2.500 & 1.906 & 0.020 & 0.017 & 21.241 & 4.039 & 31.68 & 2.16 \\
\bottomrule
\end{tabular}
}
\vspace{-2mm}
\end{table*}

\begin{table*}[!t]
\caption{Comparison with Conventional Motions and SOTA Solver of Spatiotemporal Calibration Repeatability (Mean $\pm$ STD) for Camera IMU in Real-World Experiments}
\vspace{-2mm}
\label{tab:real_run24_camera_repeatability}
\centering
\renewcommand{\arraystretch}{1.2}
\setlength{\tabcolsep}{3pt}
\resizebox{0.99\linewidth}{!}{
\begin{tabular}{cc|*{7}{c@{\pmsep}c}}
\toprule
\multicolumn{2}{c|}{\textbf{Setting}}
& \multicolumn{2}{c|}{\textbf{Time Calibration}}
& \multicolumn{6}{c|}{\textbf{Rotation Calibration}}
& \multicolumn{6}{c}{\textbf{Translation Calibration}} \\
\cmidrule(r){1-2}\cmidrule(lr){3-4}\cmidrule(lr){5-10}\cmidrule(l){11-16}
\textbf{Motion}
& \multicolumn{1}{c|}{\makebox[3.333em]{\textbf{Solver}}}
& \multicolumn{2}{c|}{$t_d$ ($\mathrm{ms}$)}
& \multicolumn{2}{c}{$R_x$ ($\mathrm{deg}$)}
& \multicolumn{2}{c}{$R_y$ ($\mathrm{deg}$)}
& \multicolumn{2}{c|}{$R_z$ ($\mathrm{deg}$)}
& \multicolumn{2}{c}{$p_x$ ($\mathrm{cm}$)}
& \multicolumn{2}{c}{$p_y$ ($\mathrm{cm}$)}
& \multicolumn{2}{c}{$p_z$ ($\mathrm{cm}$)} \\
\midrule
\multirow{2}{*}{A$^2$I-Calib++}
& A$^2$I-Calib++ & 105.532 & 0.777 & -124.375 & 0.052 & 1.468 & \bestright{0.003} & -91.738 & \bestright{0.017} & 34.611 & \bestright{0.054} & 0.878 & \bestright{0.126} & 3.378 & \secondright{0.145} \\
& iKalibr & 110.534 & \secondright{0.274} & -125.797 & 0.054 & 2.206 & 0.030 & -92.320 & 0.196 & 35.114 & \bestright{0.054} & 1.399 & \secondright{0.251} & 4.974 & \bestright{0.081} \\
\cmidrule(lr){1-2}
\multirow{2}{*}{w/o NIT}
& A$^2$I-Calib++ & 105.370 & 0.940 & -124.102 & 0.069 & 1.577 & 0.047 & -91.769 & \secondright{0.023} & 32.163 & 0.619 & -2.290 & 0.358 & 0.086 & 0.326 \\
& iKalibr & 106.452 & \bestright{0.179} & -126.286 & \bestright{0.013} & 2.355 & \secondright{0.027} & -91.884 & 0.039 & 35.217 & 0.338 & 0.292 & 0.368 & 5.203 & 0.347 \\
\cmidrule(lr){1-2}
\multirow{2}{*}{Walking}
& A$^2$I-Calib++ & 103.309 & 1.906 & -124.285 & 0.346 & 1.371 & 1.122 & -96.528 & 0.797 & 20.304 & 1.150 & -3.803 & 2.758 & -4.828 & 0.364 \\
& iKalibr & 104.205 & 2.142 & -125.772 & \secondright{0.016} & 1.227 & 0.060 & -89.303 & 1.447 & 39.616 & 2.453 & 2.878 & 1.683 & 3.447 & 0.474 \\
\bottomrule
\end{tabular}}
\vspace{-5mm}
\end{table*}

\begin{table}[!t]
\caption{Comparison of Spatiotemporal Calibration Accuracy (RMSE $\pm$ STD) in Simulation Experiments}
\vspace{-2mm}
\label{tab:sim_sota_transposed}
\centering
\renewcommand{\arraystretch}{1.25}
\setlength{\tabcolsep}{2.5pt}
\resizebox{\linewidth}{!}{
\begin{tabular}{cc|c c@{\pmsep}c c@{\pmsep}c c@{\pmsep}c}
\toprule
\makebox[3.333em]{\textbf{Solver}} & \textbf{Evaluation}
& \textbf{SR} $\uparrow$
& \multicolumn{2}{c}{$|\delta t_d|$ $\downarrow$}
& \multicolumn{2}{c}{$\|\delta\boldsymbol{\theta}\|$ $\downarrow$}
& \multicolumn{2}{c}{$\|\delta\mathbf p\|$ $\downarrow$} \\
\midrule
A$^2$I-Calib++ & All Trials
& \bestcell{9/9}
& \bestleft{0.19} & \bestright{0.05}
& \bestleft{0.18} & \bestright{0.06}
& \secondleft{0.29} & \secondright{0.07} \\
\midrule
\multirow{2}{*}{iKalibr}
& All Trials
& \secondcell{6/9}
& 226.13 & 183.42
& 93.05 & 75.82
& 34.83 & 28.29 \\
& Succeeded Only\textsuperscript{*}
& --
& \secondleft{2.95} & \secondright{0.41}
& \secondleft{0.20} & \secondright{0.09}
& \bestleft{0.27} & \bestright{0.06} \\
\bottomrule
\end{tabular}}
\begin{minipage}{0.98\linewidth}
\footnotesize
\textsuperscript{*} Succeeded Only reports error metrics over the six successful trials.
Time offset $\delta t_d$ in (\textit{ms}), rotation $\delta\boldsymbol{\theta}$ in (\textit{deg}), and translation $\delta\mathbf p$ in (\textit{cm}).
\end{minipage}
\vspace{-5mm}
\end{table}

\vspace{-1.1mm}
\subsection{Spatiotemporal Calibration Performance}
\label{subsec:calibration_performance}

Calibration results presented in the following tables are for the camera IMU relative to the trunk IMU, with translations expressed in the trunk-IMU frame.

\vspace{-0.4mm}
\subsubsection{Ablation Study on NIT}

The executed condition numbers in Tables \ref{tab:contact_vimu_validation} and
\ref{tab:sims_ablation} differ from the theoretical
$\kappa_{\mathbf R}/\kappa_{\mathbf t}=1.000/1.872$.
Nevertheless, NIT improves conditioning over w/o NIT in both simulation
and hardware. Table \ref{tab:sims_ablation} shows the simulation pair
decreasing from $280.64/9.862$ to $1.46/1.640$, together with lower
time-offset, rotation, and translation RMSEs. Across three hardware
windows, Table \ref{tab:real_run24_camera_repeatability} shows NIT's
smallest within-motion standard deviations for all  seven Search-and-Solve
parameters. These results support
improved simulation accuracy and hardware repeatability despite the
execution gap.

\subsubsection{Ablation Study on AMP}

With NIT fixed, Table \ref{tab:contact_vimu_validation} shows less slip
and greater contact duty with AMP despite larger tracking error.
AMP learns the reference state-transition distribution, retaining
anti-noise excitation without pointwise tracking.
Table \ref{tab:sims_ablation} accordingly shows lower spatiotemporal
calibration errors. For w/o NIT versus Sinusoid, AMP does not consistently
reduce errors, because AMP reduces contact disturbances but poor conditioning
can dominate calibration errors and mask this benefit.

\subsubsection{Search-and-Solve vs. iKalibr}

Table \ref{tab:real_run24_camera_repeatability} shows comparable hardware
repeatability between the two solvers. iKalibr starts from identity rotation and zero
translation/time offset, with shared accelerometer/gyroscope weights of
$17.68/57.66$, a gravity norm of $9.81~\mathrm{m/s^2}$, and $0.5$-s
time-offset padding.
Table \ref{tab:sim_sota_transposed} reports simulation success rates of
$9/9$ for Search-and-Solve and $6/9$ for iKalibr. An iKalibr trial fails if its
$\mathrm{SO}(3)$ rotation error exceeds $5^\circ$ or translation error
exceeds $5~\mathrm{cm}$. In our experiments, the failed trials lie far beyond these limits.
To diagnose these failures, we test multiple initializations:
near-ground-truth starts recover all failed trials, whereas broad random
starts seldom do, indicating initialization sensitivity.
Diagnostics trace failures to gyro-based rotation--time coupling:
rotation can partly compensate an incorrect time offset and trap
refinement in a local optimum. Search-and-Solve
avoids coupled rotation--time initialization through time-offset
correlation search followed by rotational SVD and translational least
squares.

\begin{table}[!t]
\caption{Cross-Motion Consistency of the Estimated Parameters in Real-World Experiments}
\vspace{-2mm}
\label{tab:cross_motion_consistency}
\centering\footnotesize
\renewcommand{\arraystretch}{0.9}
\setlength{\tabcolsep}{3pt}
\begin{tabular}{@{}c|cccc@{}}
\toprule
\textbf{Solver} & \textbf{Motion} & $d_t$ (ms) & $d_R$ (deg) & $d_p$ (cm) \\
\midrule
\multirow{2}{*}{A$^2$I-Calib++} & w/o NIT & \bestcell{0.16} & \bestcell{0.30} & 5.18 \\
 & Walking & \secondcell{2.22} & 4.80 & 17.15 \\
\cmidrule(lr){1-1}
\multirow{2}{*}{iKalibr} & w/o NIT & 4.08 & \secondcell{0.68} & \bestcell{1.14} \\
 & Walking & 6.33 & 3.17 & \secondcell{4.98} \\
\bottomrule
\end{tabular}
\par\vspace{1pt}
\begin{minipage}{0.98\linewidth}
\footnotesize Relative to NIT+AMP means: $d_t$, absolute time-offset difference;
$d_R$, angle between chordal mean rotations;
$d_p$, Euclidean translation distance.
\end{minipage}
\vspace{-5mm}
\end{table}

\subsubsection{Cross-Motion Consistency and Interpretation}
\label{subsubsec:cross_motion_consistency}

Table \ref{tab:cross_motion_consistency} compares motions against each
solver's NIT+AMP mean: Search-and-Solve yields smaller temporal, comparable
rotational, and larger translational shifts than iKalibr.
Walking data suggest contact/reference mismatch:
Table \ref{tab:contact_vimu_validation} reports slip/linear-reference
proxies of $0.302/0.496~\mathrm{m/s}$ versus $0.038/0.043~\mathrm{m/s}$ for NIT+AMP.
w/o NIT data suggest conditioning-related amplification:
Table \ref{tab:contact_vimu_validation} shows smaller reference proxies but
rotational/translational condition numbers of $83.82/16.08$ versus $7.40/3.36$.
In summary, rolling-contact mismatch injects linear-velocity errors, VIMU modeling/fusion
affects reference consistency, conditioning governs amplification, and
estimation can propagate errors or introduce bias.
As consistency alone cannot establish accuracy,
we adopt downstream odometry to assess calibration performance.

\subsection{Real-World Downstream Odometry Validation}
\label{subsec:odometry_validation}

\begin{figure}[!t]
\centering
\begin{minipage}[t]{0.48\linewidth}
\vspace{0pt}
\centering
\captionof{table}{Proprioceptive Odometry Performance in Real-World Experiments}
\vspace{-2mm}
\label{tab:real_best4_odometry}
\renewcommand{\arraystretch}{1.15}
\setlength{\tabcolsep}{3.2pt}
\resizebox{\linewidth}{!}{\begin{tabular}{cccc}
\toprule
\multicolumn{2}{c|}{\textbf{Settings}}
& \multicolumn{2}{c}{\textbf{Accuracy}} \\
\cmidrule(r){1-2}\cmidrule(l){3-4}
\textbf{Motion}
& \multicolumn{1}{c|}{\makebox[3.333em]{\textbf{Solver}}}
& \textbf{APE}
& \textbf{RRE} \\
\midrule
\multirow{2}{*}{A2I++}
& A2I++ & \bestcell{0.366} & \bestcell{3.277} \\
& iKalibr & 0.370 & 4.347 \\
\cmidrule(lr){1-2}
\multirow{2}{*}{w/o NIT}
& A2I++ & 0.375 & 4.327 \\
& iKalibr & 0.370 & \secondcell{4.323} \\
\cmidrule(lr){1-2}
\multirow{2}{*}{Walking}
& A2I++ & 0.513 & 5.427 \\
& iKalibr & \secondcell{0.370} & 4.338 \\
\bottomrule
\end{tabular}
}
\vspace{-1.mm}

\begin{minipage}{\linewidth}
\footnotesize \textit{Units:} APE in $\mathrm{m}$; RRE in $^\circ$.
\end{minipage}
\end{minipage}
\hfill
\begin{minipage}[t]{0.48\linewidth}
\vspace{0pt}
\centering
\includegraphics[width=\linewidth]{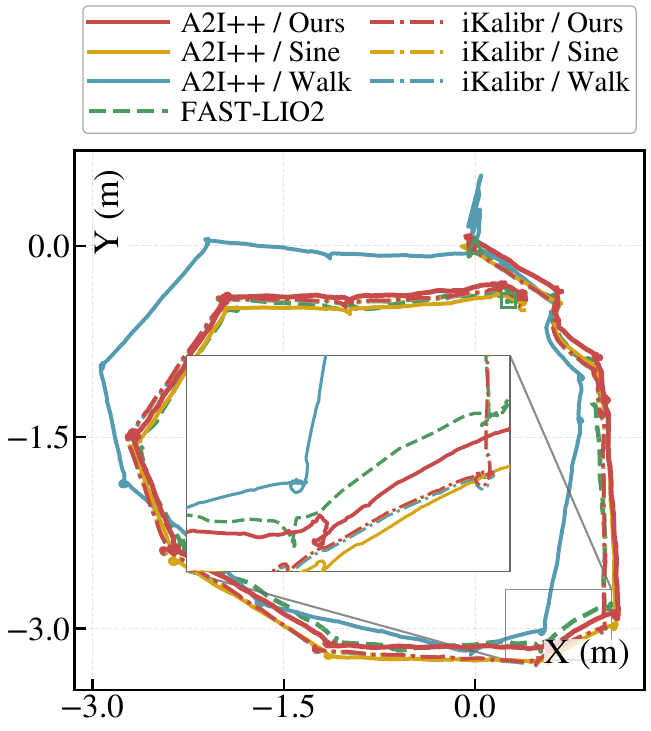}
\vspace{-7mm}
\captionof{figure}{Proprioceptive odometry trajectories compared with the FAST-LIO2 reference trajectory.}
\label{fig:sipo_six}
\end{minipage}
\vspace{-5.5mm}
\end{figure}

Against the FAST-LIO2 reference \cite{xu2022fastlio}, Table
\ref{tab:real_best4_odometry} shows that NIT+AMP with Search-and-Solve
achieves the lowest APE and RRE among the tested combinations at
$0.366~\mathrm{m}$ and $3.277^\circ$, respectively. Replacing either the motion or the
solver increases both errors, demonstrating their joint downstream
benefit. Figure \ref{fig:sipo_six} shows the corresponding trajectories.

\vspace{-2mm}
\section{Conclusion}
\vspace{-1mm}

This paper presents A$^2$I-Calib++, a kinematics-driven active calibration
framework for multi-IMU systems covered by a legged robot's kinematic model.
After the prescribed calf-IMU preliminary calibration and data collection, the
framework performs largely automated estimation of temporal offsets,
extrinsic rotations and translations, and IMU biases through an efficient
Search-and-Solve procedure. NIT provides an anti-noise reference
excitation, AMP realizes it as contact-aware robot motion, and the
contact-constrained VIMU fuses multi-leg measurements into the calibration
reference. Further evaluations show lower calibration errors in simulation, repeatable real-world calibration, and reduced downstream odometry errors. Future work will improve
adaptability to varying friction conditions and ground properties, mitigate sim-to-real gaps, and extend
the framework beyond multi-IMU calibration to exteroceptive modalities.

\vspace{-3mm}
\bibliographystyle{IEEEtran}
\bibliography{ref}

\end{document}